\documentclass{article}

\usepackage{times}
\usepackage{graphicx} 
\usepackage{subfigure} 

\usepackage{natbib}

\usepackage{algorithm}
\usepackage{algorithmic}

\usepackage{hyperref}

\usepackage[accepted]{icml2016} 
\usepackage{booktabs}
\usepackage{times}  
\usepackage{helvet}  
\usepackage{courier}  
\usepackage{url}  
\usepackage{graphicx}  
\usepackage{amssymb,graphicx,mathrsfs,epsfig,epstopdf,tikz}
\usepackage{algorithm,algorithmic}
\usepackage{amsmath}
\usepackage{amsfonts}
\usepackage{float}
\newtheorem{theorem}{Theorem}
\newtheorem{lemma}{Lemma}
\newenvironment{proof}{Proof:}{\hfill \tikz \draw[fill] (0,0) rectangle (0.25,0.25);}

\icmltitlerunning{PRIL: Perceptron Ranking Using Interval Labeled Data}

\begin{document} 
\def \bw {\tilde{\mathbf{w}}}
\def \by {\tilde{\mathbf{y}}}
\def \bx {\tilde{\mathbf{x}}}
\def \bz {\tilde{\mathbf{Z}}}
\def \X {\mathcal{X}}
\def \Y {\mathcal{Y}}
\def \xx {\mathbf{x}}
\def \I {\mathbb{I}}
\def \R {\mathbb{R}}
\def \P {\mathcal{P}}
\def \K {\mathbf{k}}
\def \yy {\mathbf{y}}
\def \zz {\mathbf{z}}
\def \ee {\mathbf{e}}
\def \vv {\mathbf{v}}
\def \ww {\mathbf{w}}
\def \dd {\mathbf{d}}
\def \thetaa  {\mbox{\boldmath $\theta$}}
\def \phii  {\mbox{\boldmath $\phi$}}
\def \I {\mathbb{I}}
\def \zero {\mathbf{0}}
\def \bI {\bar{I}}
\def \sM {\mathcal{M}}
\twocolumn[
\icmltitle{PRIL: Perceptron Ranking Using Interval Labeled Data}

\icmlauthor{Naresh Manwani}{naresh.manwani@iiit.ac.in}
\icmladdress{IIIT Hyderabad, India}

\icmlkeywords{Online Learning, Ranking, Learning with Partial Labels, Perceptron}

\vskip 0.3in
]

\begin{abstract} 
In this paper, we propose an online learning algorithm PRIL for learning ranking classifiers using interval labeled data and show its correctness. We show its convergence in finite number of steps if there exists an ideal classifier such that the rank given by it for an example always lies in its label interval. We then generalize this mistake bound result for the general case. We also provide regret bound for the proposed algorithm. We propose a multiplicative update algorithm for PRIL called M-PRIL. We provide its correctness and convergence results. We show the effectiveness of PRIL by showing its performance on various datasets.
\end{abstract} 

\section{Introduction}
Ranking (also called as ordinal classification) is an important problem in machine learning. Ranking is different from multi-class classification problem in the sense that there is an ordering among the class labels. For example, product ratings provided on online retail stores based on customer reviews, product quality, price and many other factors. Usually these ratings are numbered between 1-5. While these numbers can be thought of as class labels, there is also an ordering which has to be taken care. This problem has been very well addressed in the machine learning and referred as ordinal classification or ranking.

In general, an ordinal classifier can be completely defined by a linear function and a set of $K-1$ thresholds ($K$ be the number of classes). Each threshold corresponds to a class. Thus, the thresholds should have the same order as their corresponding classes. The classifier decides the rank (class) based on the relative position of the linear function value with respect to different thresholds. One can learn a non-linear classifier also by using an appropriate nonlinear transformation. A lot of discriminative approaches based on risk minimization principal for learning ordinal classifier has been proposed. Variants of large margin frameworks for learning ordinal classifiers are proposed in \citet{Shashua:2002,Chu:2005}. One can maintain the order of thresholds implicitly or explicitly. In the explicit way the ordering is posed as a constraint in the optimization problems itself. While the implicit method captures the ordering by posing separability conditions between every pair of classes. \citet{Li:2006} propose a generic method which converts learning an ordinal classifier into learning a binary classifier with weighted examples. A classical online algorithm for learning linear classifiers is proposed in \citet{rosenblatt}. \citet{Crammer:2001} extended Perceptron learning algorithm for ordinal classifiers. 

In the approaches discussed so far, the training data has correct class label for each feature vector. However, in many cases we may not know the exact label. Instead, we may have an interval in which the true label lies. Such a scenario is discussed in \citet{antoniuk14,Antoniuk:2016}. In this setting, corresponding to each example, an interval label is provided and it is assumed that the true label of the example lies in this interval. In \citet{Antoniuk:2016}, a large margin framework for batch learning is proposed using interval insensitive loss function. 

In this paper, we propose an online algorithm for learning ordinal classifier using interval labeled data. We name the proposed approach as {\bf PRIL} (Perceptron ranking using interval labeled data). Our approach is based on interval insensitive loss function. As per our knowledge, this is the first ever online ranking algorithm using interval labeled data. We show the correctness of the algorithm by showing that after each iteration, the algorithm maintains the orderings of the thresholds. We derive the mistake bounds for the proposed algorithm in both ideal and general setting. In the ideal setting, we show that the algorithm stops after making finite number of mistakes. We also derive the regret bound for the algorithm. We also propose a multiplicative update algorithm for PRIL (called M-PRIL). We also show the correctness of M-PRIL and find its mistake bound. 

The rest of the paper is organized as follows. In section~\ref{sec:proposed-method}, we describe the problem of learning ordinal classifier using interval labeled data. In section~\ref{sec:pril}, we discuss the proposed online algorithm for learning ordinal classifier using interval labeled data. We derive the mistake bounds and the regret bound in section~\ref{sec:analysis}. We present the experimental results in section~\ref{sec:exp}. We make the conclusions and some remarks on the future work in section~\ref{sec:conclusions}.

\section{Ordinal Classification using Interval Labeled Data}
\label{sec:proposed-method}
Let $\X \subseteq \R^d$ be the instance space. Let $\Y=\{1,\ldots,K\}$ be the label space. Our objective is to learn an ordinal classifier $h:\X\rightarrow \Y$ which has the following form 
\begin{align*}
h(\xx) & = 1+\sum_{k=1}^{K-1} \I_{\{\ww.\xx>\theta_k\}}= \min_{i\in [K-1]} \{i: \ww.\xx - b_i \leq 0\}
\end{align*}
where $\ww \in \R^d$ and $\thetaa \in \R^{K-1}$ be the parameters to be optimized. Parameters $\thetaa = [\theta_1 \;\ldots\;\theta_{K-1}]$ should be such that $\theta_1 \leq \theta_2 \leq \;\ldots \;\leq \theta_{K-1}$. The classifier splits the real line into $K$ consecutive intervals using thresholds $\theta_1,\ldots,\theta_{K-1}$ and then decides the class label based on which interval corresponds to $\ww.\xx$.

Here, we assume that for each example $\xx$, the annotator provides an interval $[y_l,y_r] \in \Y\times \Y$ ($y_l \leq y_r$). The interval annotation means that the true label $y$ for example $\xx$ lies in the interval $[y_l,y_r]$. Let $S=\{(\xx^1,y_l^1,y_r^1),\ldots,(\xx^T,y_l^T,y_r^T)\}$ be the training set. 

Discrepancy between the predicted label and corresponding label interval can be measured using {\em interval insensitive loss} \citep{Antoniuk:2016}. 
\begin{equation}
\label{eq:MAE_Loss}
L_I^{MAE} (f(\xx),\thetaa,y_l,y_r) = \sum_{i=1}^{y_l-1} \I_{\{f(\xx) < \theta_i\}} + \sum_{i=y_r}^{K-1} \I_{\{f(\xx) \geq \theta_i\}}
\end{equation}
Where subscript $I$ stands for {\em interval}. This, loss function takes value $0$, whenever $\theta_{y_l-1}\leq f(\xx)\leq \theta_{y_r}$. However, this loss function is discontinuous. A
convex surrogate of this loss function is as follows \citep{Antoniuk:2016}:
\begin{align}
\nonumber L_I^{IMC}(f(\xx),y_l,y_r,\thetaa) &= \sum_{i=1}^{y_l-1} \max\left(0,-f(\xx)+\theta_i\right)\\
& +\sum_{i=y_r}^{K-1}\max\left(0,f(\xx)-\theta_i\right)\label{surrogate_loss}
\end{align}
Here $IMC$ stands for the implicit constraints for ordering of thresholds $\theta_i$s. For a given example-interval pair $\{\xx,(y_l,y_r)\}$, the loss $L_I^{IMC}(f(\xx),\thetaa,y_l,y_r)$ becomes  zero only when
\begin{eqnarray*}
f(\xx)-\theta_i \geq 0 & \forall i \in \{1,\ldots, y_l-1\}\\
f(\xx)-\theta_i \leq 0 & \forall i \in \{y_r,\ldots, K-1\}
\end{eqnarray*}
Note that if for any $i\in \{1,\ldots,y_l^t-1\}$, $f(\xx) < \theta_i$, then $f(\xx)<\theta_j,\;\forall j=y_r^t,\ldots,K-1$ because $\theta_1\leq \theta_2 \leq \ldots \leq \theta_{K-1}$. Similarly, if for any $i\in \{y_r^t,\ldots,K-1\}$, $f(\xx) > \theta_i$, then $f(\xx)>\theta_j,\;\forall j=1,\ldots,y_l^t-1$. Let $\bI=\{1,\ldots,y_l-1\}\cup \{y_r,\ldots,K-1\}$. Then, we define $z_i, \;\forall i\in \bI$ as follows. 
\begin{eqnarray}
\label{eq:dummy_labels}
z_i = \begin{cases}
+1 & \forall i \in \{1,\ldots, y_l-1\}\\
-1 & \forall i \in \{y_r,\ldots, K-1\}
\end{cases}
\end{eqnarray}
Thus, $L_I^{IMC}(f(\xx),y_l,y_r,\thetaa) =0$ requires that $z_i(f(\xx)-\theta_i)\geq 0,\;\forall i \in \bI$. Thus, $L_I^{IMC}$ can be re-written as:
\begin{eqnarray}
\nonumber L_I^{IMC}(f(\xx),y_l,y_r,\thetaa) = \sum_{i\in \bI} \max\left( 0,-z_i\left(f(\xx)-\theta_i\right)\right) 
\end{eqnarray}

\section{Perceptron Ranking using Interval Labeled Data}
\label{sec:pril}
In this section, we propose an online algorithm for ranking using {\em interval insensitive loss} described in eq.~(\ref{surrogate_loss}). Our algorithm is based on stochastic gradient descent on $L_I^{IMC}$. 

We derive the algorithm for linear classifier. Which means, $f(\xx)=\ww.\xx$. Thus, the parameters to be estimated are $\ww$ and $\thetaa$. We initialize with $\ww^0=\mathbf{0}$ and $\thetaa^0=\mathbf{0}$. Let $\ww^t,\thetaa^t$ be the estimates of the parameters in the beginning of trial $t$. Let at trial $t$, $\xx^t$ be the example observed and $[y_l^t,y_r^t]$ be its label interval.
$\ww^{t+1}$ and $\thetaa^{t+1}$ are found as follows.
\begin{align*}
\ww^{t+1} &= \ww^t - \eta\nabla_{\ww} L_I^{IMC}(\ww.\xx^t,y_l^t,y_r^t,\thetaa)\big{|}_{\ww^t,\thetaa^t}\\
&=\ww^t + \sum_{i\in \bI^t}z_i^t\xx^t\I_{\{z_i^t(\ww^t.\xx^t - \theta_i^t)<0\}} \\
\theta^{t+1}_i &= \theta^t_i - \eta\frac{\partial  L_I^{IMC}(\ww.\xx^t,y_l^t,y_r^t,\thetaa)}{\partial \theta_i}\big{|}_{\ww^t,\thetaa^t}\\
&= \begin{cases}\theta_i^t - z_i^t\I_{\{z_i^t(\ww^t.\xx^t - \theta_i^t)<0\}} & \forall i \in \bI^t\\
\theta_i^t & \forall i \notin \bI^t
\end{cases}
\end{align*}
Thus, only those constraints will participate in the update which are not satisfied. The violation of $i$th constraint leads to the update contribution of $z_i^t\xx^t$ in $\ww^{t+1}$ and $z_i^t$ in $\theta_i^{t+1}$. $\theta_i^t,\;t\notin \bI^t$ are not updated in trial $t$. The complete approach is described in Algorithm~\ref{algo2}.
\begin{algorithm}[h]
\caption{{\bf P}erceptron {\bf R}anking using  {\bf I}nterval {\bf L}abeled Data (PRIL)}
\label{algo2}
\begin{algorithmic}
\STATE {\bf Input: } Training Dataset $\mathcal{S}$\;
\STATE {\bf Initialize} Set $t=1$, $\ww_1=\zero$, $\theta_1^{1}=\theta_2^{1}=\ldots=\theta_{K-1}^1=0$, $m=1$\;
\FOR{$i\leftarrow 1$ to $T$}
\STATE Get example $\xx_t$ and its $(y_l^t,y_r^t)$
\FOR{$i\leftarrow 1$ to $y_l^t-1$}
\STATE $z_i^{t}=+1$
\ENDFOR
\FOR{$i\leftarrow y_r^t$ to $K-1$}
\STATE $z_i^{t}=-1$
\ENDFOR
\STATE Initialize $\tau_i^t=0,\;i\in [K-1]$
\FOR{$i\in \bI_t$}
\IF{$z_i^t(\ww^t.\xx^t -\theta_i^t)\leq 0$}
\STATE $\tau_i^t = z_i^t$
\ENDIF
\ENDFOR
\STATE $\ww^{t+1} = \ww^t + \sum_{i=1}^{K-1}\tau_i^t \xx^t$\;
\STATE $\theta_i^{t+1}=\theta_i^t -\tau_i^t,\; i=1\ldots K-1$\;
\ENDFOR
\STATE {\bf Output}: $h(\xx)=\min_{i\in [K]}\big{\{}i\;:\;\ww^{T+1}.\xx-\theta_i^{T+1} <0 \big{\}}$
\end{algorithmic}
\end{algorithm}
It is important to see that when exact labels are given to the Algorithm~\ref{algo2} instead of partial labels, it becomes same as the algorithm proposed in \cite{Crammer:2001}. PRIL can be easily extended for learning nonlinear classifiers using kernel methods. 

\subsection{Kernel PRIL}
We can easily extend the proposed algorithm PRIL for learning nonlinear classifiers using kernel functions. We see that the classifier learnt after $t$ trials using PRIL can be completely determined using $\tau_i^s,\;i\in \{1,\ldots,K-1\},s\in [t]$ as follows
$\ww^{t+1} = \sum_{s=1}^t\sum_{\bI^s}\tau_i^s\xx^s$ and $\theta^{t+1}_i = -\sum_{s=1}^t \tau_i^s,\;i=1\ldots K-1$.
Also, $f^{t+1}(\xx)$ can be found as 
$f^{t+1}(\xx) = \sum_{s=1}^t\sum_{i\in \bI^s}\tau_i^s \xx^s.\xx$.
Thus, we can replace the inner product with a suitable kernel function $\kappa:\X\times\X\longrightarrow \R$ and
represent $f^{t+1}(\xx)$ as
$$f^{t+1}(\xx) = \sum_{s=1}^t\sum_{i\in \bI^s}\tau_i^s \kappa(\xx^s,\xx)$$
Similarly, $\theta_i^{t+1}$ can be expressed as $\theta_i^{t+1}=\sum_{s=1}^t\tau_i^t,\;\forall i \in \bI^t$. The ordinal classifier learnt after $T$ trials is
$$h(\xx) = \min_{i\in [K]} \left\{i\;:\;\sum_{s=1}^T\left(\sum_{j \in \bI^s} \tau_i^s\kappa(\xx^s,\xx)\right) + \tau_i^s <0\right\}$$
Complete description of kernel PRIL is provided in Algorithm~\ref{algo3}.

\begin{algorithm}[h]
\caption{Kernel PRIL}
\label{algo3}
\begin{algorithmic}
\STATE {\bf Input: }Training Dataset $\mathcal{S}$\;
\STATE {\bf Output: }$\tau_1^t,\ldots,\tau_{K-1}^t,\;t=1\ldots T$\;
\STATE {\bf Initialize: } Set $\tau_1^0=\ldots=\tau_{K-1}^0=0$, $f^0(.)=0$ and $t=0$\;
\FOR{$t\leftarrow 1$ to $N$}
\STATE Get example $\xx^t$ and its $(y_l^t,y_r^t)$;
\FOR{$i\leftarrow 1$ to $y_l^t-1$}
\STATE $z_i^{t}=+1$;
\ENDFOR
\FOR{$i\leftarrow y_r^t$ to $K-1$}
\STATE $z_i^{t}=-1$;
\ENDFOR
\STATE Initialize $\tau_i^t=0,\;\forall i \in [K]$;
\FOR{$i\in \bI_t$}
\IF{$z_i^t\big{(}f^t(\xx^t) -\theta_i^t\big{)}\leq 0$}
\STATE $\tau_i^t = z_i^t$;
\ENDIF
\ENDFOR
\STATE $f^{t+1}(.) = f^t(.) + \sum_{i \in \bI^t}\tau_i^t \kappa(\xx^t,.)$
\STATE $\theta_i^{t+1} = \theta_i^t - \tau_i^t,\;\forall i \in [K]$
\ENDFOR
\STATE $h(\xx) = \min_{i \in [K]} \{i\;:\; f^{T+1}(\xx) -\theta_i^t<0\}$
\end{algorithmic}
\end{algorithm}

\subsection{Analysis}
\label{sec:analysis}
Now we will show that PRIL preserves the ordering of the thresholds $\theta_1,\ldots,\theta_{K-1}$.

\begin{lemma}{\bf Order Preservation: }
\label{lemma:pril1}
Let $\ww^t$ and $\thetaa^t=[\theta_1^t\ldots\theta_{K-1}^t]^T$ be the current parameters for ranking classifier, where $\theta_1^t \leq \theta_2^t\leq \ldots\leq\theta_{K-1}^t$. Let $\xx^t$ be the instance fed to PRIL at trial $t$ and $[y_l^t,y_r^t]$ be its corresponding rank interval. Let $\ww^{t+1}$ and $\thetaa^{t+1}=[\theta_1^{t+1}\ldots\theta_{K-1}^{t+1}]^T$ be the resulting ranking classifier parameters after the update of PRIL. Then, $\theta_1^{t+1} \leq \theta_2^{t+1}\leq \ldots\leq\theta_{K-1}^{t+1}$.       
\end{lemma}
\begin{proof}
Note that $\theta_t^i\in\mathbb{Z},\forall i \in \{1,\ldots, K-1\},\forall t\in \{1,\ldots, N\}$ as PRIL initializes $\theta_i^1=0,\forall i\in \{1,\ldots,K-1\}$.
To show that PRIL preserves the ordering of the thresholds, we consider following different cases.
\begin{enumerate}
\item $i\in\{1,\ldots,y_l^t-2\}$: we see that,
 \begin{align*}
 \theta_{i+1}^{t+1}& - \theta_i^{t+1} 
 =  \theta_{i+1}^t - \theta_i^t - z_{i+1}^t\I_{\{z_{i+1}^t(\ww^t.\xx^t-\theta_{i+1}^t)\leq 0\}}\\
 &+ z_i^t\I_{\{z_i^t(\ww^t.\xx^t-\theta_i^t)\leq 0\}}\\
 &= \theta_{i+1}^t - \theta_i^t +\I_{\{\ww^t.\xx^t-\theta_i^t\leq 0\}}-\I_{\{\ww^t.\xx^t-\theta_{i+1}^t\leq 0\}}
 \end{align*} 
We used the fact that $z_i^t = +1,\forall i \in \{1,\ldots,y_l^t-1\}$. Thus, there can be two cases only.
 \begin{enumerate}
 \item $\theta_{i+1}^t = \theta_i^t $: In this case, we simply get $\theta_{i+1}^{t+1} = \theta_i^{t+1}$.
 \item $\theta_{i+1}^t > \theta_i^t $: Since $\theta_i^t\in\mathbb{Z}\;\forall i,\forall t$, we get $\theta_{i+1}^t \geq  \theta_i^t +1 $. This means
 \begin{eqnarray*}
  \theta_{i+1}^{t+1} - \theta_i^{t+1} \geq 1+\I_{\{\ww^t.\xx^t-\theta_i^t\leq 0\}}-\I_{\{\ww^t.\xx^t-\theta_{i+1}^t\leq 0\}}
 \end{eqnarray*}
 But, $\I_{\{\ww^t.\xx^t-\theta_i^t\leq 0\}}-\I_{\{\ww^t.\xx^t-\theta_{i+1}^t\leq 0\}} \in\{-1,0\}$. Thus,
 $\theta_{i+1}^{t+1} - \theta_i^{t+1} \geq 0$.
 \end{enumerate}  
 \item $i=y_l^t-1$: In this case $\theta_{i+1}^{t+1}=\theta_{i+1}^t$ as per the update rule. Also, $z_i^t=+1$. Thus, using the fact that
 $\theta_{i+1}^t - \theta_i^t \geq 0$, we get:
 \begin{eqnarray*}
\theta_{i+1}^{t+1} - \theta_i^{t+1} &=&  \theta_{i+1}^t - \theta_i^t + z_i^t\I_{\{z_i^t(\ww^t.\xx^t-\theta_i^t)\leq 0\}}\\
 &=& \theta_{i+1}^t - \theta_i^t +\I_{\{\ww^t.\xx^t-\theta_i^t\leq 0\}}\geq 0
 \end{eqnarray*}
 \item $i\in \{y_l^t,\ldots,y_r^t-2\}$: PRIL does not update thresholds in this range. Thus,
$\theta_i^{t+1} = \theta_i^t,\;\forall i\in \{y_l^t,\ldots,y_r^t-1\}$. Thus, $\theta_{i+1}^{t+1}=\theta_{i+1}^t\geq\theta_i^{t+1}=\theta_i^t,\;\forall i \in \{y_l^t,\ldots,y_r^t-2\}$.
 \item $i=y_r^t-1$: In this case $\theta_{i}^{t+1}=\theta_{i}^t$ as per the update rule. Also, $z_{i+1}^t=-1$. Thus, using the fact that
 $\theta_{i+1}^t - \theta_{i}^t \geq 0$, we get:
 \begin{align*}
\theta_{i+1}^{t+1} - \theta_{i}^{t+1} &=  \theta_{i+1}^t - \theta_{i}^t - z_{i+1}^t\I_{\{z_{i+1}^t(\ww^t.\xx^t-\theta_{i+1}^t)\leq 0\}}\\
 &= \theta_{i+1}^t - \theta_{i}^t +\I_{\{\ww^t.\xx^t-\theta_{i+1}^t\geq 0\}}\geq 0
 \end{align*}
 \item $i\in\{y_r,\ldots,K-1\}$: we see that,
 \begin{align*}
 \theta_{i+1}^{t+1} &- \theta_i^{t+1} =  \theta_{i+1}^t - \theta_i^t - z_{i+1}^t\I_{\{z_{i+1}^t(\ww^t.\xx^t-\theta_{i+1}^t)\leq 0\}}\\
 & + z_i^t\I_{\{z_i^t(\ww^t.\xx^t-\theta_i^t)\leq 0\}}\\
 &= \theta_{i+1}^t - \theta_i^t -\I_{\{\ww^t.\xx^t-\theta_i^t\geq 0\}}+\I_{\{\ww^t.\xx^t-\theta_{i+1}^t\geq 0\}}
 \end{align*}
We used the fact that $z_i^t = -1,\forall i \in \{y_r^t,\ldots,K-1\}$. Thus, there can be two cases only.
 \begin{enumerate}
 \item $\theta_{i+1}^t = \theta_i^t $: In this case, we simply get $\theta_{i+1}^{t+1} = \theta_i^{t+1}$.
 \item $\theta_{i+1}^t > \theta_i^t $: Since $\theta_i^t\in\mathbb{Z}\;\forall i,\forall t$, we get $\theta_{i+1}^t \geq  \theta_i^t +1 $. This means $
  \theta_{i+1}^{t+1} - \theta_i^{t+1} \geq 1+\I_{\{\ww^t.\xx^t-\theta_{i+1}^t\geq 0\}}-\I_{\{\ww^t.\xx^t-\theta_i^t\geq 0\}}$.
 But, $\I_{\{\ww^t.\xx^t-\theta_{i+1}^t\geq 0\}}-\I_{\{\ww^t.\xx^t-\theta_i^t\geq 0\}} \in\{0, -1\}$. Thus,
 $\theta_{i+1}^{t+1} - \theta_i^{t+1} \geq 0$.
 \end{enumerate}  
\end{enumerate}
\end{proof}

Now we will show that the PRIL makes finite number of mistakes if there exists an ideal interval ranking classifier. 
\begin{theorem}\label{thm1}
Let $S=\{(\xx^1,y_l^1,y_r^1),\ldots,(\xx^T,y_l^T,y_r^T)\}$ be an input sequence. Let $R_2^2 = \max_{t\in [T]}\;||\xx^t||^2$ and $c=\min_{t\in [T]}(y_r^t-y_l^t)$. Let
$\exists \gamma >0$, $\ww^*\in\R^d$ and $\theta_1^*,\ldots,\theta_{K-1}^* \in \R$ such that $||\ww^*||^2_2 + \sum_{i=1}^{K-1}(\theta_i^*)^2 = 1$ and $\min_{i\in \bI^t} z_i^t(\ww^*.\xx^t-\theta_i^*)\geq \gamma,\;\forall t\in [T]$.
 Then,
 \begin{equation*}
 \sum_{t=1}^TL_\text{I}^{\text{MAE}}(f^t(\xx^t),\thetaa,y_l^t,y_r^t) \leq \frac{(R^2_2+1)(K-c-1)}{\gamma^2}
 \end{equation*}
 where $f^t(\xx^t) = \ww^t.\xx^t$.
\end{theorem}
\begin{proof}
Let $\vv^*=[{\ww^*}^T\;\;\theta_i^*\;\;\theta_2^*\;\;\ldots\;\;\theta_{K-1}^*]^T$ and $\vv^t = [(\ww^t)^T\;\;\theta_i^t\;\;\theta_2^t\;\;\ldots\;\;\theta_{K-1}^t]^T$.
Let Algorithm~\ref{algo2} makes a mistake at trial $t$. Let $\sM^t=\{i\in \bI^t\;|\; z_i^t(\ww^t.\xx^t-\theta_i^t)\leq 0 \}$. $\sM^t$ be the set of indices of the
constraints which are not satisfied at trial $t$. Let $m^t=|\sM^t|$ be the number of those constraints. Thus,
\begin{align*}
\vv^{t+1}.\vv^* 
&=(\ww^t+\sum_{i=1}^{K-1}\tau_i^t\xx^t).\ww^* + \sum_{i=1}^{K-1}(\theta_i^t-\tau_i^t)\theta_i^*\\
&= \vv^t.\vv^* + \sum_{i=1}^{K-1}\tau_i^t(\ww^*.\xx^t-\theta_i^*)\\
&= \vv^t.\vv^* + \sum_{i \in \sM^t}z_i^t(\ww^*.\xx^t-\theta_i^*)\\
&\geq \vv^t.\vv^* + \sum_{i \in \sM^t} \gamma = \vv^t.\vv^* + \gamma m^t
\end{align*}
where we have used the fact that $z_i^t(\ww^*.\xx^t-\theta_i^*)\geq \gamma$. Summing over both sides from $t=1$ to $T$ and using $\vv^1=\mathbf{0}$, we get:
\begin{equation}
\label{eq:lb}
\vv^{T+1}.\vv^* \geq \vv^1.\vv^*+ \gamma \sum_{t=1}^T m^t = \gamma\sum_{t=1}^T m^t
\end{equation}
Now, we will upper bound $||\vv^{T+1}||^2_2$. 
\begin{align*}
||\vv^{t+1}||^2_2 -||\vv^t||^2_2 &=
2\sum_{i=1}^{K-1}\tau_i^t \ww^t.\xx^t +(\sum_{i=1}^{K-1} \tau_i^t)^2||\xx^t||^2_2 \\
&\;\;\;\;-2\sum_{i=1}^{K-1}\tau_i^t\theta_i^t + \sum_{i=1}^{K-1}(\tau_i^t)^2\\
&= ||\vv^t||^2_2 + 2\sum_{i\in \sM^t}\tau_i^t (\ww^t.\xx^t-\theta_i^t) \\
&\;\;\;\;+(\sum_{i\in\sM^t} \tau_i^t)^2||\xx^t||^2_2+ \sum_{i\in \sM^t}(\tau_i^t)^2 
\end{align*}
We know that $||\xx^t||^2_2\leq R_2^2 \;\forall t;\;\; \sum_{i\in \sM^t}(\tau_i^t)^2 =m^t;\;\;(\sum_{i\in \sM^t}\tau_i^t)^2 \leq (m^t)^2$
and $z_i^t(\ww^t.\xx^t -\theta_i^t) \leq 0\;\forall i \in \sM^t$. Thus, 
$||\vv^{t+1}||^2_2 - ||\vv^t||^2_2 \leq(m^t)^2R^2_2 + m^t$.
Summing over both sides from $t=1$ to $T$ and using $\vv^1 = \mathbf{0}$, we get,
\begin{equation}
\label{eq:ub}
||\vv^{T+1}||^2_2 \leq  R_2^2\sum_{t=1}^T(m^t)^2 + \sum_{t=1}^T m^t 
\end{equation}
Now, using Cauchy-Schwartz inequality, we get $\vv^{T+1}.\vv^* \leq ||\vv^{T+1}||_2.||\vv^*||_2=||\vv^{T+1}||_2$.
Now using eq.~(\ref{eq:lb}) and (\ref{eq:ub}), we get
\begin{align*}
& \gamma^2 (\sum_{t=1}^Tm^t)^2 \leq  ||\vv^{T+1}||_2^2 \leq R_2^2\sum_{t=1}^T(m^t)^2 + \sum_{t=1}^Tm^t\\
& \Rightarrow  \sum_{t=1}^Tm^t \leq \frac{1}{\gamma^2} \left( R_2^2 \frac{\sum_{t=1}^T(m^t)^2}{\sum_{t=1}^T m^t} + 1\right)
\end{align*}
But, $m^t \leq K - (y_r^t-y_l^t)-1 \leq K-c-1$, then $\sum_{t=1}^T(m^t)^2 \leq \sum_{i=1}^Tm^t(K-c-1)$. Which gives
\begin{eqnarray*}
\sum_{t=1}^Tm^t  \leq  \frac{R_2^2 (K-c-1) +1}{\gamma^2}\leq \frac{(R^2_2 +1)(K-c-1)}{\gamma^2}
\end{eqnarray*}
But $m^t
=L_I^{MAE}(f^t(\xx^t),\thetaa^t,y_l^t,y_r^t)$
Which means,
 \begin{equation*}
 \sum_{t=1}^TL_I^{MAE}(f^t(\xx^t),\thetaa,y_l^t,y_r^t) \leq \frac{(R^2_2+1)(K-c-1)}{\gamma^2}
 \end{equation*}
 \end{proof}
 
In Theorem~\ref{thm1}, we assumed that there exists an ideal classifier defined by $\ww^*\in \R^d$ and $\thetaa^*$. Let $\vv^*=[{\ww^*}'\;\;{\thetaa^*}']'$ and $f^*(\xx^t)=\ww^*.\xx^t$. Thus, 
$$L_I^{MAE}(f^*(\xx^t),\thetaa^*,y_l^t,y_r^t)=0,\;\forall t\in [T]$$
Which means, $z_i^t(\ww^*.\xx^t-\theta_i^*)\geq 0,\;\forall i \in \bI^t,\forall t\in [T]$ where $z_i^t,\;i\in\bI^t$ are as described in eq.~(\ref{eq:dummy_labels}). Now we define $\hat{\xx}^t_i \in \R^{d+K-1},\;\forall i \in \bI^t$ as follows.
\begin{equation}
\label{newx}
\hat{\xx}^t_i = [(\xx^t)'\;\;0\;\;\ldots\;\;1\;\;0\;\;\ldots\;\;0]',\;i\in \bI^t
\end{equation}
 where the component values at locations $d+1,\ldots,d+K-1$ are all set to '0' except for the location $(d+i)$.  Component value at location $(d+i)$ is set to '-1' in $\hat{\xx}^t_i$. Thus, we have $z_i^t (\vv^*.\hat{\xx}^t_i) \geq 0,\forall i \in \bI^t,\forall t\in [T]$. 
Thus, $\vv^*$ correctly classifies all the $\hat{\xx}^t_i,\;\forall i \in \bI^t,\;\forall t\in [T]$.
However, in general, for a given dataset, we may not know if such an ideal classifier exists. Next, we derive the mistake bound for this general setting. 

\begin{theorem}
 Let $S=\{(\xx^1,y_l^1,y_r^1),\ldots,(\xx^T,y_l^T,y_r^T)\}$ be an input sequence. Let $\gamma >0$, $R_2^2 = \max_{t\in [T]}\;||\xx^t||^2$ and $c = \min_{t\in [T]}(y_r^t-y_l^t)$.
Thus, for any $\ww \in\R^d$, $\thetaa\in \R^{K-1}$ such that $||\ww||^2_2 + ||\thetaa||_2^2 = 1$, we get
 \begin{equation*}
 \sum_{t=1}^TL_I^{MAE}(f^t(\xx^t),\thetaa,y_l^t,y_r^t) \leq  \frac{(D+\sqrt{R_2^2+1})^2(K-c-1)}{\gamma^2}
 \end{equation*}
 where $f^t(\xx^t) = \ww^t.\xx^t$, $D^2 = \sum_{t=1}^T\sum_{i\in \bI^t}(d_i^t)^2$ and $d_i^t=\max[0,\gamma-z_i^t(\ww.\xx^t-\theta_i)],\;i\in \bI^t,\;\forall t\in [T]$.
\end{theorem}
\begin{proof}
We proceed by constructing $\tilde{\xx}^t_i\in \mathbb{R}^{d+(T+1)(K-1)},\;\forall i\in\bI^t$ corresponding to every $\xx^t,\;t\in [T]$ as follows. $$\tilde{\xx}^t_i=[(\hat{\xx}^t_i)'\;0\;\ldots\;0\;\Delta\;0\;\ldots\;0]'$$ where $\hat{\xx}^t_i$ is as described in eq.(\ref{newx}).  The first $d+K-1$ components of 
$\tilde{\xx}^t_i$ are same as $\hat{\xx}^t_i$ and rest of all the elements are set to 0 except for the location $d+(K-1)t+ i$, which is set to $\Delta$. Let $\mathbf{u}^t\in \mathbb{R}^{K-1}$ be as follows.
$$\mathbf{u}^t=\frac{\left[ z_1^td^t_1\;\ldots \;z_{y_l^t-1}^t d^t_{y_l^t-1}\;\;0\;\ldots  \;0\;\;z_{y_r^t}^td^t_{y_r^t}\;\ldots \;z_{K-1}^td^t_{K-1}\right]}{\Delta Z}$$
where $d_i^t=\max\left[ 0,\gamma-z_i^t(\ww.\xx^t-\theta_i)\right],\;i\in \bI^t,\;\forall t\in [T]$.
We now construct $\tilde{\mathbf{v}}' \in \mathbb{R}^{d+(T+1)(K-1)}$ as follows. 
$$\tilde{\mathbf{v}} = \left[\frac{\ww'}{Z}\;\;\frac{\thetaa'}{Z}\;\;\mathbf{u}^1\;\;\mathbf{u}^2\;\;\ldots\;\;\mathbf{u}^T\right]'$$ 
 $Z$ is chosen such that $||\tilde{\mathbf{v}}||_2^2=1$. Thus, 
$$Z^2=1+\frac{\sum_{t=1}^N\sum_{i\in \bI^t}(d_i^t)^2}{\Delta^2}=1+\frac{D^2}{\Delta^2}$$
We also see that $||\tilde{\xx}^t_i||^2_2 = ||\xx^t||^2_2 +  1+\Delta^2 \leq R_2^2+\Delta^2+1,\;\forall i \in \bI^t,\;\forall t \in [T]$. Moreover,
\begin{eqnarray*}
z_i^t\tilde{\mathbf{v}}.\tilde{\xx}^t_i &=& \frac{z_i^t(\ww.\xx^t-\theta_i)}{Z} + \frac{d_i^t}{Z}\\
&\geq & \frac{z_i^t(\ww.\xx^t-\theta_i)}{Z} + \frac{\gamma - z_i^t(\ww.\xx^t-\theta_i)}{Z}=\frac{\gamma}{Z}
\end{eqnarray*}
Thus, $\tilde{\mathbf{v}}$ correctly classifies all the examples in $\mathbb{R}^{d+(T+1)(K-1)}$ with margin at least $\frac{\gamma}{Z}$. Thus, by using the mistake bound given in Theorem~\ref{thm1}, we get
$$\sum_{t=1}^T L_I^{MAE}(f^t(\xx^t),\thetaa^t,y_l^t,y_r^t)\leq \frac{Z^2(R_2^2+\Delta^2+1)K_1}{\gamma^2}$$
where $K_1=K-c-1$. To find the least upper bound, we minimize the RHS above with respect to $\Delta$. For minimum $\Delta^* = (D^2(R_2^2+1))^{\frac{1}{4}}$.
Replacing $\Delta$ by $\Delta^*$, we get
$$\sum_{t=1}^T L_I^{MAE}(f^t(\xx^t),\thetaa^t,y_l^t,y_r^t)\leq \frac{(D+\sqrt{R_2^2+1})^2(K-c-1)}{\gamma^2}$$
\end{proof}

\subsubsection*{Regret Analysis}
Let $S=\{(\xx^1,y_l^1,y_r^1),\ldots,(\xx^T,y_l^T,y_r^T)\}$ be the input sequence. Let $\{\ww^t,\thetaa^t\}_{t=1}^T$ be the sequence of parameter vectors generated by an online ranking algorithm $\mathcal{A}$. Then, regret of algorithm $\mathcal{A}$ is defined as 
\begin{align*}
R_T(\mathcal{A}) & = \sum_{t=1}^T L_I^{IMC}\left(\ww^t.\xx^t,\thetaa^t,y_l^t,y_r^t\right) \\
& \;\;- \min_{(\ww,\thetaa)\in \Omega}\;\sum_{t=1}^T L_I^{IMC}\left( \ww.\xx^t,\thetaa,y_l^t,y_r^t\right)
\end{align*}
For online gradient descent applied on convex cost functions, the regret bound analysis is given by \citet{Zinkevich:2003}. We know that the objective function of PRIL is also convex. Motivated by that, here we find the regret bound for PRIL.
\begin{theorem}
\label{thm:pril2}
Let $S=\{(\xx^1,y_l^1,y_r^1),\ldots,(\xx^T,y_l^T,y_r^T)\}$ be an input sequence such that $\xx^t \in \R^d,\;\forall t \in [T]$. Let $R_2^2 = \max_{t\in [T]}\;||\xx^t||^2$ and $c = \min_{t\in [T]}(y_r^t-y_l^t)$. Let $\Omega =\{\vv \in \R^{d+K-1}\;:\;||\vv||_2\leq \Lambda\}$. Let $\vv^t=(\ww^t,\thetaa^t),\;t=1\ldots T+1$ be the sequence of vectors in $\Omega$ such that $\forall t\geq 1$, $\vv^{t+1} = \vv^t -\nabla_t$ where $\nabla_t$ belongs to the sub-gradient set of $L_I^{IMC}\left(\ww.\xx^t,\thetaa,y_l^t,y_r^t\right)$ at $\vv^t$. Then,
\begin{equation*}
 R_T(\text{PRIL}) \leq \frac{1}{2}\left[\Lambda^2 + T(R^2_2 +1)(K-c-1)\right]
\end{equation*}
\end{theorem} 
\begin{proof}
Let $\vv=(\ww,\thetaa)\in \Omega$. Using the convexity property of $L_I^{IMC}$, we get
\begin{align*}
& L_I^{IMC}\left( \ww^t.\xx^t,\thetaa^t,y_l^t,y_r^t \right)- L_I^{IMC}\left( \ww.\xx^t,\thetaa,y_l^t,y_r^t\right)\\
& \leq \nabla L_I^{IMC}\left( \ww^t.\xx^t, \thetaa^t,y_l^t,y_r^t\right). (\vv^t - \vv)\\
&= (\vv^t-\vv^{t+1}).(\vv^t-\vv)\;\;\;\;\; \text{using the update in Algorithm~\ref{algo2}} \\
& = \frac{1}{2}\left[ ||\vv-\vv^t||^2 -||\vv-\vv^{t+1}||^2 + ||\vv^t-\vv^{t+1}||^2 \right]\\
& = \frac{1}{2}\left[ ||\vv-\vv^t||^2 -||\vv-\vv^{t+1}||^2\right]  \\
& \;\;\;\;+\frac{1}{2}||\nabla L_I^{IMC}\left( \ww^t.\xx^t, \thetaa^t,y_l^t,y_r^t\right)||^2 
\end{align*}
But,
\begin{align*}
& ||\nabla L_I^{IMC}\left( \ww^t.\xx^t, \thetaa^t,y_l^t,y_r^t\right)||^2\\
&= ||\sum_{i \in \bI^t}\tau_i^t \xx^t||^2 + ||\sum_{i\in \bI^t}\tau_i^t||^2\\
& =(||\xx^t||^2 +1) ||\sum_{i\in \bI^t}\tau_i^t||^2 \leq (R_2^2+1)\left( \sum_{i\in \bI^t}|\tau_i^t|\right)\\
&\leq (R_2^2 +1) (K-y_r^t + y_l^t -1)\leq (R^2_2 +1)(K-c-1)
\end{align*}
Thus,
\begin{align}
\nonumber & L_I^{IMC}\left( \ww^t.\xx^t,\thetaa^t,y_l^t,y_r^t \right)- L_I^{IMC}\left( \ww.\xx^t,\thetaa,y_l^t,y_r^t\right) \\
& \leq \frac{1}{2}\big{[} ||\vv-\vv^t||^2 -||\vv-\vv^{t+1}||^2+ (R^2_2 +1)(K-c-1) \big{]}\label{eq:diffeq}
\end{align}
Summing eq.(\ref{eq:diffeq}) on both sides from 1 to $T$, we get,
\begin{align*}
& \sum_{t=1}^T L_I^{IMC}\left( \ww^t.\xx^t,\thetaa^t,y_l^t,y_r^t \right)- \sum_{t=1}^T L_I^{IMC}\left( \ww.\xx^t,\thetaa,y_l^t,y_r^t\right) \\
&\leq \frac{1}{2}\big{[} ||\vv-\vv^1||^2 -||\vv-\vv^{T+1}||^2+ T(R^2_2 +1)(K-c-1) \big{]}\\
&\leq \frac{1}{2}\big{[} ||\vv||^2 + T(R^2_2 +1)(K-c-1) \big{]}
\end{align*}
where we have used the fact that $||\vv-\vv^{T+1}||^2 \geq 0$. We know that $\min_{\vv \in \Omega}||\vv||_2^2 =\Lambda^2$. Since the above holds for any $\vv=(\ww,\thetaa) \in \Omega$, we have
\begin{align*}
R_T(\text{PRIL})
\leq \frac{1}{2}\left[\Lambda^2 + T(R^2_2 +1)(K-c-1) \right]
\end{align*}
\end{proof}

\section{Multiplicative PRIL}
In this section, we propose a multiplicative algorithm for PRIL called {\em M-PRIL}. In this algorithm, the weight vector $\ww$ and the thresholds $\thetaa$ are modified in a multiplicative manner. The algorithm is inspired from the Winnow algorithm proposed for learning linear predictors \cite{KIVINEN19971}. In M-PRIL, at every iteration $t$, the weight vector $\ww^t$ and the thresholds $\thetaa^t$ are maintained such that $\|\ww^t\|^2+ \|\thetaa^t\|^2 = 1$.
The complete algorithm M-PRIL is described in Algorithm~\ref{algo:mpril}. 
\begin{algorithm}[h]
\caption{M-PRIL}
\label{algo:mpril}
\begin{algorithmic}
\STATE {\bf Input: } Training Dataset $\mathcal{S}$\;
\STATE {\bf Initialize} $t=1$, $\ww_1=\zero$, $\theta_1^{1}=\ldots=\theta_{K-1}^1=0$\;
\FOR{$i\leftarrow 1$ to $T$}
\STATE Get example $\xx^t$ and its $(y_l^t,y_r^t)$
\FOR{$i\leftarrow 1$ to $y_l^t-1$}
\STATE $z_i^{t}=+1$
\ENDFOR
\FOR{$i\leftarrow y_r^t$ to $K-1$}
\STATE $z_i^{t}=-1$
\ENDFOR
\STATE Initialize $\tau_i^t=0,\;i\in [K-1]$
\FOR{$i\in \bI^t$}
\IF{$z_i^t(\ww^t.\xx^t -\theta_i^t)\leq 0$}
\STATE $\tau_i^t = z_i^t$
\ENDIF
\ENDFOR
\STATE $\mathcal{Z}^t = \sum_{i=1}^d w_i^te^{\eta x_i^t\sum_{k=1}^{K-1}\tau_k^t} + \sum_{i=1}^{K-1}\theta_i^te^{-\eta\sum_{k=1}^{K-1}\tau_k^t}$
\FOR{$i \in [d]$}
\STATE $w_i^{t+1} =\frac{1}{\mathcal{Z}^t} w_i^t e^{\eta x_i^t\sum_{k=1}^{K-1}\tau_k^t}$
\ENDFOR
\FOR{$i \in [K-1]$}
\STATE $\theta_i^{t+1}=\frac{1}{\mathcal{Z}^t}\theta_i^te^{-\eta\tau_i^t}$
\ENDFOR
\ENDFOR
\STATE {\bf Output}: $h(\xx)=\min_{i\in [K]}\big{\{}i\;:\;\ww^{T+1}.\xx-\theta_i^{T+1} <0 \big{\}}$
\end{algorithmic}
\end{algorithm}

\begin{lemma}
\label{lemma:mpril1}
{\bf Order Preservation: } Let $\ww^t \in \R^d$ and $\thetaa^t \in \R^{K-1}$ denote the current parameters of current ranking classifier. Assume that $\theta_1^t \leq \theta_2^t \leq \ldots \leq \theta_t^{K-1}$. Let $\ww^{t+1}$ and $\thetaa^{t+1}$ be the new parameters generated by the M-PRIL algorithm after observing $(\xx^t,y_l^t,y_r^t)$. Then, $\theta_1^{t+1} \leq \theta_2^{t+1} \leq \ldots \leq \theta_{t+1}^{K-1}$. 
\end{lemma}
\begin{proof}
Note that $\theta_t^i\in\mathbb{Z},\forall i \in \{1,\ldots, K-1\},\forall t\in \{1,\ldots, N\}$ as M-PRIL initializes $\theta_i^1=\frac{1}{K+d-1},\forall i\in \{1,\ldots,K-1\}$.
To show that M-PRIL preserves the ordering of the thresholds, we consider following different cases.
\begin{enumerate}
\item $i\in\{1,\ldots,y_l^t-2\}$: we see that,
 \begin{align*}
 \theta_{i+1}^{t+1} - \theta_i^{t+1} 
 &=  \frac{1}{\mathcal{Z}^t}\left(\theta_{i+1}^te^{-\eta \tau_{i+1}^t} - \theta_i^t e^{-\eta \tau_i^t}\right)
 \end{align*}
 We know that $\tau_i^t = z_i^t\I_{\{z_i^t(\ww^t.\xx^t-\theta_i^t)\leq 0\}}$ and $z_i^t = +1,\forall i \in \{1,\ldots,y_l^t-1\}$. Thus, there can be two cases only.
 \begin{enumerate}
 \item $\theta_{i+1}^t = \theta_i^t $: In this case, we simply get $\theta_{i+1}^{t+1} = \theta_i^{t+1}$.
 \item $\theta_{i+1}^t > \theta_i^t $: 
 We see that 
 \begin{align*}
 &\I_{\{\ww^t.\xx^t-\theta_i^t\leq 0\}} \leq \I_{\{\ww^t.\xx^t-\theta_{i+1}^t\leq 0\}}\\
 \Rightarrow & -\eta \tau_{i+1}^t \geq - \eta \tau_i^t\\
 \Rightarrow & e^{-\eta \tau_{i+1}^t} \geq e^{- \eta \tau_i^t}\\
 \Rightarrow & \theta_{i+1}^t e^{-\eta \tau_{i+1}^t} \geq \theta_i^t e^{- \eta \tau_i^t}
 \end{align*}
 Thus,
 $\theta_{i+1}^{t+1} - \theta_i^{t+1} \geq 0$.
 \end{enumerate}  
 \item $i=y_l^t-1$: In this case $\tau_{i+1}^t=0$. Thus, using the fact that
 $\theta_{i+1}^t - \theta_i^t \geq 0$, we get:
 \begin{align*}
\theta_{i+1}^{t+1} - \theta_i^{t+1} &=  \frac{1}{\mathcal{Z}^t}\left(\theta_{i+1}^t - \theta_i^t e^{-\eta\tau_i^t}\right) \\
& \geq \frac{1}{\mathcal{Z}^t}\left(\theta_{i+1}^t - \theta_i^t \right) \;\;\;\;\;\because \tau_i^t \in \{0,1\}\\
& \geq  0
 \end{align*}
 \item $i\in \{y_l^t,\ldots,y_r^t-2\}$: In this case, we see that $\tau_i^t = 0,\;\forall i\in \{y_l^t,\ldots,y_r^t-2\}$ . Thus,
$\theta_i^{t+1} = \frac{1}{\mathcal{Z}^t}\theta_i^t,\;\forall i\in \{y_l^t,\ldots,y_r^t-2\}$. Thus, $\theta_{i+1}^{t+1}=\frac{1}{\mathcal{Z}^t}\theta_{i+1}^t\geq\frac{1}{\mathcal{Z}^t}\theta_i^{t}=\theta_i^{t+1},\;\forall i \in \{y_l^t,\ldots,y_r^t-2\}$.
 \item $i=y_r^t-1$: In this case $\tau_i^t=0$. Also, $z_{i+1}^t=-1$. Thus, using the fact that
 $\theta_{i+1}^t - \theta_{i}^t \geq 0$, we get:
 \begin{align*}
\theta_{i+1}^{t+1} - \theta_{i}^{t+1} &=\frac{1}{\mathcal{Z}^t}\left(  \theta_{i+1}^te^{-\eta\tau_{i+1}^t} - \theta_{i}^t \right)\\
 &\geq  \frac{1}{\mathcal{Z}^t}\left(  \theta_{i+1}^t - \theta_{i}^t \right) \;\;\;\;\;\because \tau_{i+1}^t \in \{-1,0\}\\
 \geq 0
 \end{align*}
 \item $i\in\{y_r,\ldots,K-1\}$: we see that,
 \begin{align*}
 \theta_{i+1}^{t+1} - \theta_i^{t+1} &=  \frac{1}{\mathcal{Z}^t}\left(\theta_{i+1}^te^{-\eta\tau_{i+1}^t} - \theta_i^te^{-\eta\tau_{i}^t}\right)
 \end{align*}
We used the fact that $z_i^t = -1,\forall i \in \{y_r^t,\ldots,K-1\}$. Thus, there can be two cases only.
 \begin{enumerate}
 \item $\theta_{i+1}^t = \theta_i^t $: In this case, we simply get $\tau_{i+1}^t = \tau_i^t$. Thus,
 \begin{align*}\theta_{i+1}^{t+1} - \theta_i^{t+1}& = \frac{1}{\mathcal{Z}^t}\left( \theta_{i+1}^{t}e^{-\eta\tau_{i+1}^t} - \theta_i^{t}e^{-\eta\tau_i^t}\right)=0
 \end{align*}
 \item $\theta_{i+1}^t > \theta_i^t $: We see that
 \begin{align*}
 &\I_{\{\ww^t.\xx^t-\theta_i^t\leq 0\}} \leq \I_{\{\ww^t.\xx^t-\theta_{i+1}^t\leq 0\}}\\
 \Rightarrow &1- \I_{\{\ww^t.\xx^t-\theta_{i+1}^t\geq 0\}} \leq 1-\I_{\{\ww^t.\xx^t-\theta_{i}^t\geq 0\}}\\
  \Rightarrow & \I_{\{\ww^t.\xx^t-\theta_{i}^t\geq 0\}} \leq \I_{\{\ww^t.\xx^t-\theta_{i+1}^t\geq 0\}}\\
   \Rightarrow & \I_{\{z_{i}^t(\ww^t.\xx^t-\theta_{i}^t)\leq 0\}} \leq \I_{\{z_{i+1}^t(\ww^t.\xx^t-\theta_{i+1}^t)\leq 0\}}\\
 \Rightarrow &  \tau_{i+1}^t \leq  \tau_i^t\\
 \Rightarrow & e^{-\eta \tau_{i+1}^t} \geq e^{- \eta \tau_i^t}\\
 \Rightarrow & \theta_{i+1}^t e^{-\eta \tau_{i+1}^t} \geq \theta_i^t e^{- \eta \tau_i^t}
 \end{align*}
 Thus,
 $\theta_{i+1}^{t+1} - \theta_i^{t+1} \geq 0$.
 \end{enumerate}  
\end{enumerate}

\end{proof}

Thus, M-PRIL preserves the order of thresholds in consecutive rounds. We now find the mistake bound for M-PRIL. 

\begin{theorem}
Let $S=\{(\xx^1,y_l^1,y_r^1),\ldots,(\xx^T,y_l^T,y_r^T)\}$ be an input sequence to M-PRIL algorithm. Let $\|\xx^t\|_\infty \leq 1,\;\forall t \in [T]$ and $c=\min_{t\in [T]}(y_r^t-y_l^t)$. Let
$\exists \gamma >0$, $\ww^*\in\R^d$ and $\theta_1^*,\ldots,\theta_{K-1}^* \in \R$ such that $\|\ww^*\|_1  + \|\thetaa^*\|_1 = 1$ and $\min_{i\in \bI^t} z_i^t(\ww^*.\xx^t-\theta_i^*)\geq \gamma,\;\forall t\in [T]$.
 Then, the ranking loss of M-PRIL is 
 \begin{equation*}
 \sum_{t=1}^TL_\text{I}^{\text{MAE}}(f^t(\xx^t),\thetaa^t,y_l^t,y_r^t) \leq (K-c-1)^2\frac{\log (K+d-1)}{\gamma^2}
 \end{equation*}
 where $f^t(\xx^t) = \ww^t.\xx^t$.
\end{theorem}

\begin{proof}
Let $\vv^t = (\ww^t,\thetaa^t)$. We start by finding the decrease in the KL-divergence between $\vv^t$ and $\vv^*$. Thus,
\begin{align*}
\Delta_t &= D_{KL}(\vv^*||\vv^{t+1}) - D_{KL}(\vv^*||\vv^t) 
\end{align*}
The logic is to bound $\sum_{t \in [T]} \Delta_t$ from above and below. We first derive the upper bound. 
\begin{align*}
\Delta_t &= \sum_{i=1}^d w_i^*\log\left(\frac{w_i^{t+1}}{w_i^t}\right) + \sum_{k=1}^{K-1}\theta_k^*\log\left(\frac{\theta_k^{t+1}}{\theta_k^t}\right)\\
&=\sum_{i=1}^d w_i^*\log\left(\frac{\mathcal{Z}^t}{e^{\eta x_i^t\sum_{k=1}^{K-1}\tau_k^t}}\right) + \sum_{k=1}^{K-1}\theta_k^*\log\left(\frac{\mathcal{Z}^t}{e^{-\eta \tau_k^t}}\right)\\
&=\log\left(\mathcal{Z}^t\right)\left[\sum_{i=1}^d w_i^* + \sum_{k=1}^{K-1}\theta_k^*\right]-\eta \sum_{k=1}^{K-1}\tau_k^t(\ww^*.\xx^t-\theta_k^*)\\
&=\log\left(\mathcal{Z}^t\right)-\eta \sum_{k \in \mathcal{M}^t}z_k^t(\ww^*.\xx^t-\theta_k^*)\\
&\leq \log\left(\mathcal{Z}^t\right)-\eta \gamma m^t\;\;\;\because z_k^t(\ww^*.\xx^t - \theta^*_k)\geq \gamma,\;\forall k \in \bI^t
\end{align*}
Let $c_1=K-c-1$. Note that $\left|x_i^t \sum_{k=1}^{K-1}\tau_k^t\right| \leq c_1$. We now bound $\log\left(\mathcal{Z}^t\right)$ as follows:
\begin{align*}
\mathcal{Z}^t & = \sum_{i=1}^d w_i^te^{\eta x_i^t\sum_{k\in \mathcal{M}^t}\tau_k^t} + \sum_{i=1}^{K-1}\theta_i^te^{-\eta\tau_k^t}\\
& \leq  \sum_{i=1}^d w_i^t\left[ \frac{1+\frac{x_i^t\sum_{k=1}^{K-1}\tau^t_k}{c_1}}{2}e^{\eta c_1} +\frac{1-\frac{x_i^t\sum_{k=1}^{K-1}\tau^t_k}{c_1}}{2}e^{-\eta c_1}\right]\\
&\;\;\;\;\;+\sum_{k=1}^{K-1} \theta_k^t\left[ \frac{1+\frac{\tau_k^t}{c_1}}{2}e^{\eta c_1} + \frac{1-\frac{\tau_k^t}{c_1}}{2}e^{-\eta c_1}\right]\\
& =  \frac{1}{2}\left(e^{\eta c_1} +e^{-\eta c_1}\right) \left[\sum_{i=1}^d w_i^t +\sum_{k=1}^{K-1}\theta_k^t \right]\\
&\;\;\;\;\;+\sum_{k=1}^{K-1}  \frac{e^{\eta c_1} -e^{-\eta c_1 }}{2c_1}\tau_k^t(\ww^t.\xx^t-\theta_k^t)\\
& \leq  \frac{1}{2}\left(e^{\eta c_1} +e^{-\eta c_1}\right) \;\;\;\because \tau_k^t(\ww^t.\xx^t-\theta_k^t)\leq 0,\; \forall k
\end{align*}
Thus, $\Delta_t$ becomes
\begin{align*}
\Delta_t & \leq \log\left[\frac{1}{2}\left(e^{\eta(K-c-1)} + e^{-\eta(K-c-1)}\right)\right]-\eta\gamma m^t\\
& \leq m^t\left\{\log\left[\frac{1}{2}\left(e^{\eta(K-c-1)} + e^{-\eta(K-c-1)}\right)\right]-\eta\gamma\right\}
\end{align*}
In the above we used the fact that $m^t \geq 1$. Now, summing $\Delta_t$ over $t=1$ till $T$, we get
\begin{align*}
\sum_{t=1}^T\Delta_t
& \leq \log\left[\frac{1}{2}\left(e^{\eta(K-c-1)} + e^{-\eta(K-c-1)}\right)\right]\sum_{t=1}^T m^t\\
& \;\;\;\;-\eta\gamma\sum_{t=1}^T m^t
\end{align*}
Now, we find the lower bound on $\sum_{t=1}^T\Delta_t$ as follows:
\begin{align*}
\sum_{t=1}^T \Delta_t & =\sum_{t=1}^T \left( D_{KL}(\vv^*\| \vv^{t}) - D_{KL}(\vv^*\| \vv^{t+1})\right)\\
& = D_{KL}(\vv^*\| \vv^{1}) - D_{KL}(\vv^*\| \vv^{T+1})\\
&\geq - D_{KL}(\vv^*\| \vv^{1}) \;\;\;\;\;\because D_{KL}(\vv^*\| \vv^{T+1})\geq 0\\
& \geq -\log(K+d-1)
\end{align*}
Now comparing the lower and the upper bound on $\sum_{t=1}^T \Delta_t$, we get
\begin{align}
\sum_{t=1}^T m^t \leq \frac{\log(K+d-1)}{\left\{\log\left[\frac{2}{e^{\eta(K-c-1)} + e^{-\eta(K-c-1)}}\right]+\eta\gamma\right\}}\label{mistake-bound}
\end{align}
The RHS above can be minimized by taking $\eta$ as
\begin{align*}
\eta = \frac{1}{2(K-c-1)} \log\left( \frac{(K-c-1) + \gamma}{(K-c-1) - \gamma}\right)
\end{align*}
By substituting this value of $\eta$ in eq.(\ref{mistake-bound}), we get
\begin{align*}
\sum_{t=1}^T m^t \leq \frac{K+d-1}{g\left(\frac{\gamma}{(K-c-1)}\right)}
\end{align*}
where $g(t) = \frac{1+t}{2}\log(1+t) + \frac{1-t}{2}\log(1-t)$. It can be shown that $g(t) \geq \frac{t^2}{2}$ for $0\leq t \leq 1$. We know that \begin{align*}
 \gamma & \leq z_k^t(\ww^*.\xx^t-\theta_k^*) \;\;\; \text{by the definition of } (\ww^*,\thetaa^*)\\
& =z_k^t \vv^*.\hat{\xx}_k^t\;\;\;\;\;\hat{\xx}_k^t \text{ is defined in eq.(\ref{newx})}\\
& \leq \|\vv^*\|_1 \|\hat{\xx}_k^t\|_\infty\;\;\;\;\;\text{using Cauchy-Schwarz Inequality}\\
& =\max(\|\xx^t\|_\infty,1)=1\;\;\;\;\;\because \|\xx^t\|_\infty =1
\end{align*}
Thus, $\frac{\gamma}{K-c-1} \leq 1$. By putting the above inequality in $g(t)$, we get
\begin{align*}
\sum_{t=1}^T L_I^{MAE}&(f^t(\xx^t),\thetaa^t,y_l^t,y_r^t)  = \sum_{t=1}^T m^t \\
& \leq (K-c-1)^2 \frac{\log(K+d-1)}{\gamma^2}
\end{align*}
\end{proof}
Thus, M-PRIL would converge after making finite number of mistakes if there exists an ideal interval ranking classifier for the given training data.

\section{Experiments}
\label{sec:exp}
We now discuss the experimental results to show the effectiveness of the proposed approach. We first describe the datasets used.

\subsection{Dataset Description}
We show the simulation results on 3 datasets.  The description of these datasets are as given below.
\begin{enumerate}
\item {\bf Synthetic Dataset: }We generate points $\xx \in \R^2$ uniformly at random from the unit square $[0,1]^2$. For each point, the rank was assigned from the set $\{1,\ldots,5\}$ as $y= \max_r \{r\;:\;10(x_1- 0.5)(x_2- 0.5)+ \xi > b_r\}$
where $\mathbf{b}= (-\infty,-1,-0.1,0.25,1)$. $\xi\sim \mathcal{N}(0,0.125)$ (normally distributed with zero mean and a standard deviation of 0.125). The visualization of the synthetic dataset is provided in Figure~\ref{fig1}. We generated 100 sequences of instance-rank pairs for synthetic dataset each of length 10000.
\begin{figure}[h]
\begin{center}
\includegraphics[scale=0.28]{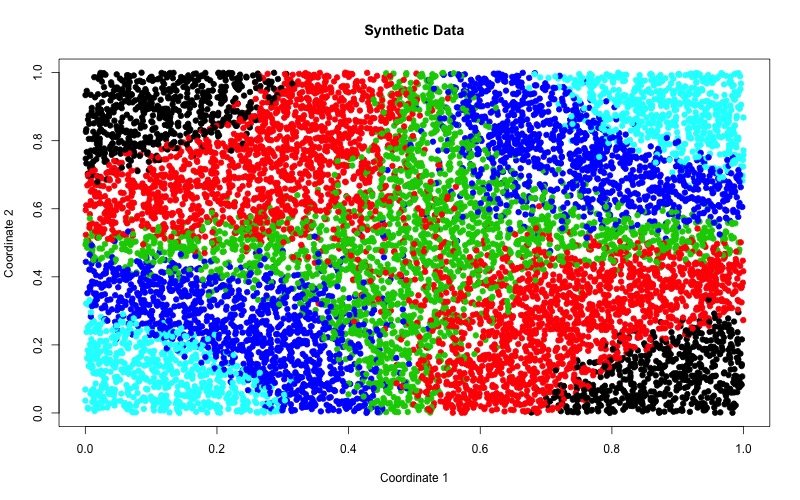}
\caption{Synthetic Dataset}
\label{fig1}
\end{center}
\end{figure}
\begin{figure*}[t]
\begin{tabular}{ccc}
\includegraphics[scale=0.2]{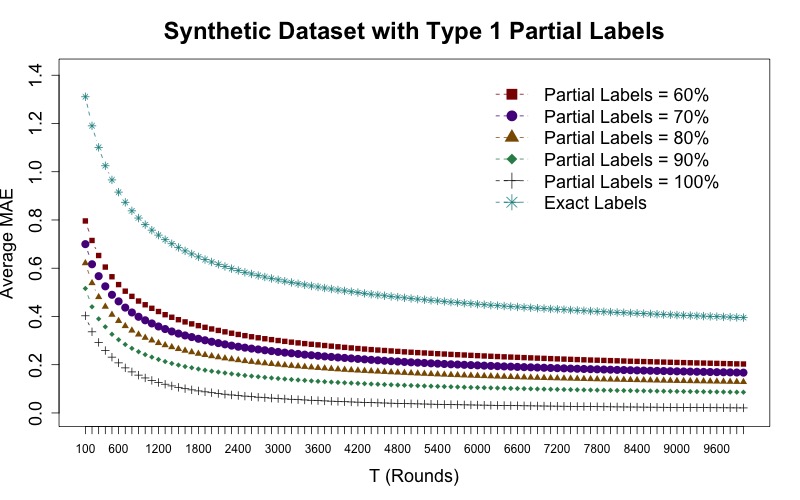} &
\includegraphics[scale=0.2]{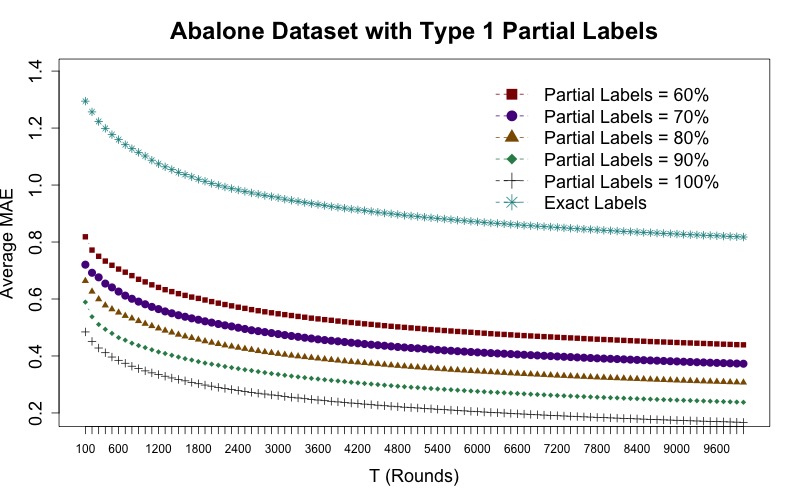} &
\includegraphics[scale=0.2]{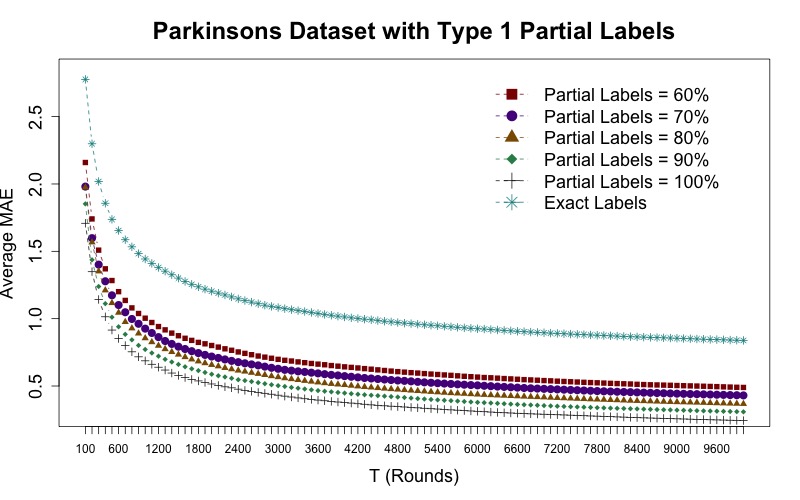}\\
\includegraphics[scale=0.2]{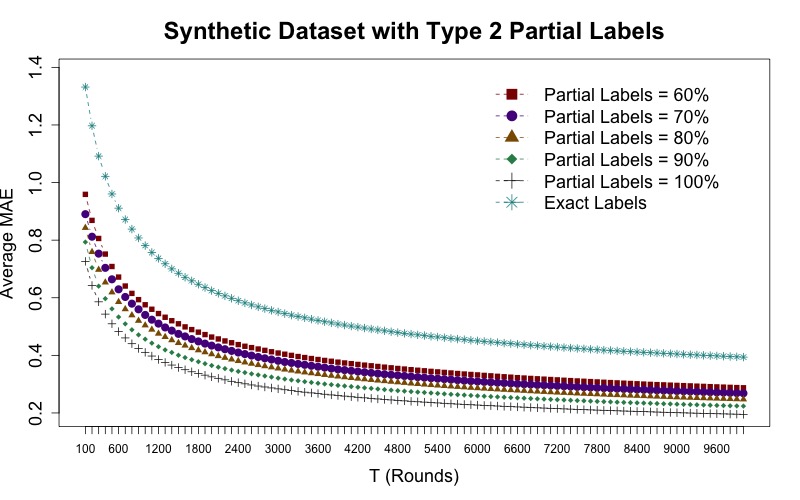}&
\includegraphics[scale=0.2]{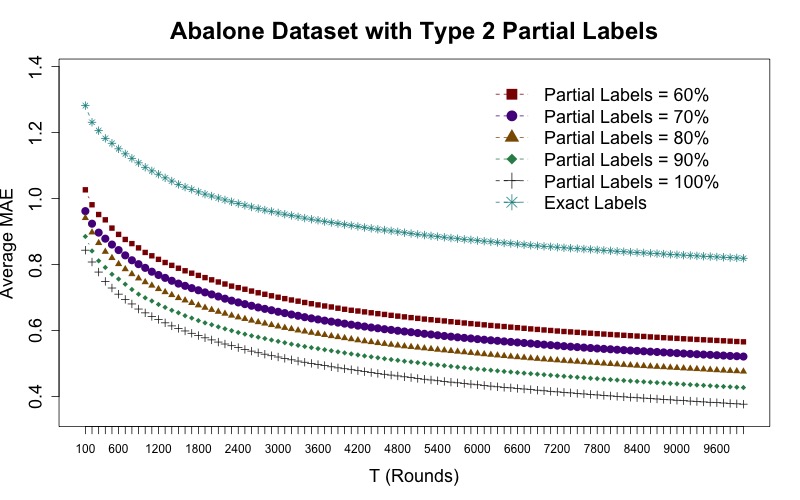} &
\includegraphics[scale=0.2]{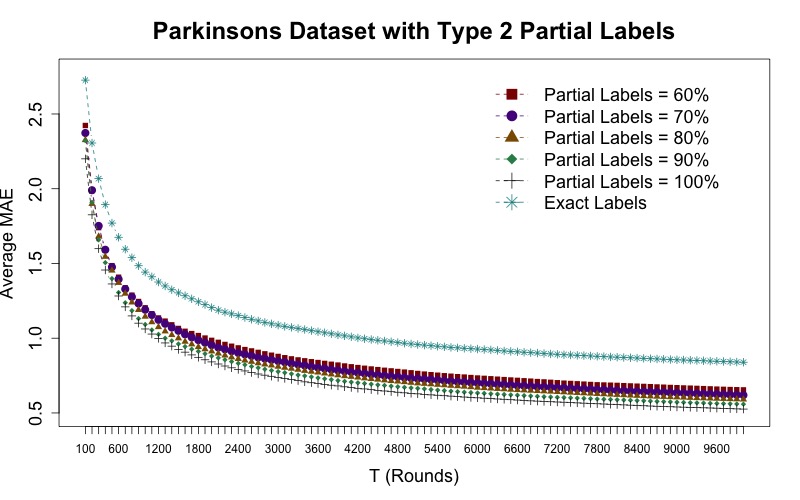}
\end{tabular}
\caption{Experiment: Varying the fraction of examples with partial labels. Variation in average MAE performance by varying the fraction of partial labels in the training set. The loss is computed using the partial labels.}
\label{fig:exp1}
\end{figure*}
\item {\bf Parkinsosns Telemonitoring Dataset: }This dataset \citep{Lichman:2013} contains biomedical voice measurements for 42 patients in various stages of Parkinson’s disease. There are total 5875 observations. There are 20 features and the target variables is total UPDRS. In the dataset, the total-UPDRS spans the range 7-55, with higher values representing more severe disability. We divide the range of total-UPDRS into 10 parts. We normalize each feature independently by making the mean 0 and standard deviation 1.
\item {\bf Abalone Dataset: }The age of
abalone \citep{Lichman:2013} is determined by counting the number of rings. There are 8 features and 4177 observations. The number of rings vary from 1 to 29. However, the distribution is very skewed. So, we divide the whole range into 4 parts, namely 1-7, 8-9, 10-12, 13-29. 
\end{enumerate}

The training data is comprised of two parts. One part contains the absolute labels for feature vectors and the other contains the partial labels. In our experiments, we keep 25\% of the examples which have correct label and rest 75\% examples having partial labels. 

{\bf Generating Interval Labels: }
Let there be $K$ categories. We consider two different methods of generating interval labels.
\begin{enumerate}
\item For $y\in\{2,\ldots,(K-1)\}$, the interval label was randomly chosen between
$[y-1,y]$ and $[y,y+1]$ where $y$ is the true label. For the class label $1$, the interval label was set to $[1,2]$. For class label $K$, the interval label was set to $[K-1,K]$.
\item For $y\in\{2,\ldots,(K-1)\}$, the interval labels were set to the interval $[y-1,y+1]$ where $y$ is the true label. For the class label $1$, the interval label was set to $[1,2]$. For class label $K$, the interval label was set to $[K-1,K]$.
\end{enumerate}

\subsection{Experimental Setup}
{\bf Kernel Functions Used: } We used following kernel functions for different datasets.
\begin{itemize}
\item {Synthetic: }$\kappa(\xx_1,\xx_2) = (\xx_1.\xx_2 + 1)^2$.
\item { Parkinson's Telemonitoring: }$\kappa(\xx_1,\xx_2) = \xx_1.\xx_2$.
\item { Abalone: }$\kappa(\xx_1,\xx_2) = (\xx_1.\xx_2 + 1)^3$.
\end{itemize}

\begin{figure}[h]
\includegraphics[scale=0.3]{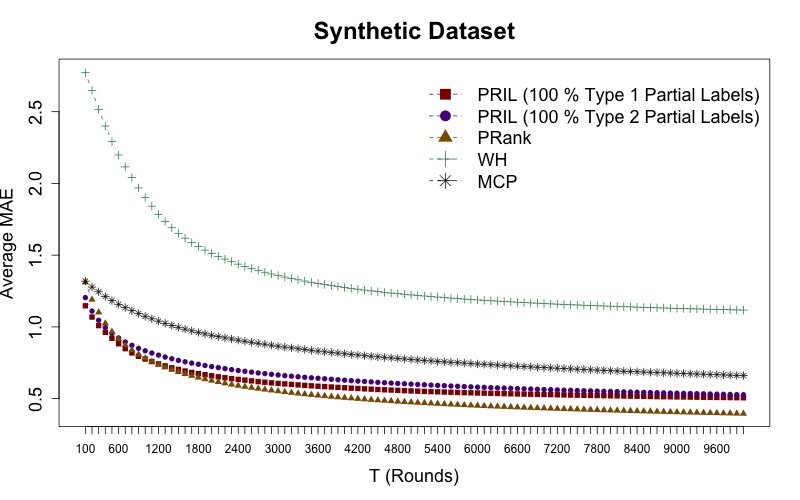}\\
\includegraphics[scale=0.3]{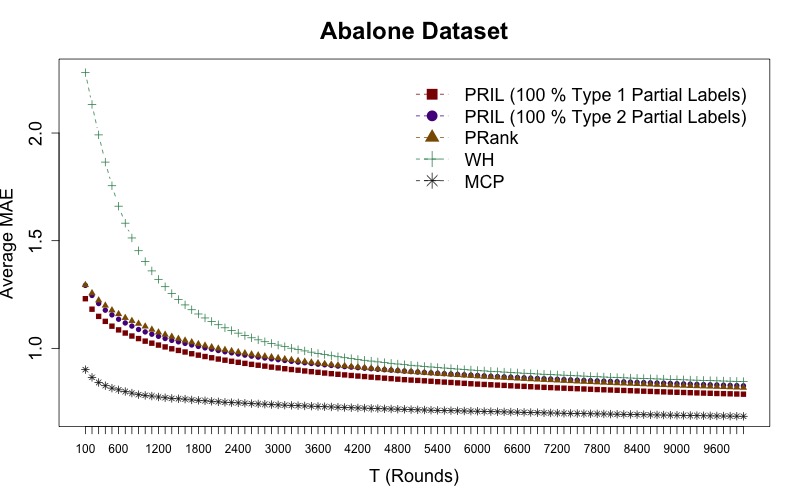}\\
\includegraphics[scale=0.3]{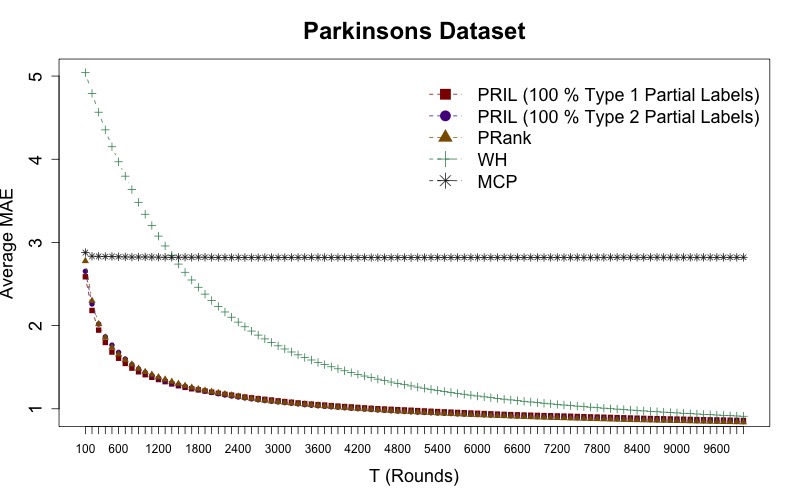}
\caption{PRIL vs. PRank, Widrow-Hoff and Multi-Class Perceptron (MCP). PRIL trained using partial labels and evaluated on exact labels.}
\label{fig:exp2}
\end{figure}
\subsection{Varying the Fraction of Partial Labels in PRIL}
We now discuss the performance of PRIL when we vary the fraction of partial labels in the training. We train with 60\%, 70\%, 80\%, 90\% and 100\% examples with partial labels. We compute the loss at each round with the same partial label used for updating the hypothesis. For PRIL, at each time step, we compute the average of $L_{MAE}$ (defined in eq.(\ref{eq:MAE_Loss})). The results are shown in Figure~\ref{fig:exp1}. We see that for all the datasets the average mean absolute error (MAE) is decreases faster as compared to the number of rounds ($T$). Also, the average MAE decreases with the increase in the fraction of examples with partial labels. This we observe in all 3 datasets and both types of partial labels. This happens because the allowed range for predicted values is more for partial labels as compared to the exact label. Consider an example $\xx^t$ with partial labels as $[y_l^t,y_r^t]$ ($y_r^t>y_l^t$) and exact label as $y^t$. The MAE for this example with partial labels is sum of $K-y_r^t+y_l^t-1$ losses. On the other hand, MAE becomes sum of $K-1$ losses if the label if exact for $\xx^t$. Thus, as we increase the fraction of partial labels in the training set, the average MAE decreases.
We see that training with no partial labels gets larger values of average MAE for all the datasets.

\subsection{Comparisons With Other Approaches}
We compare the proposed algorithm with (a) PRank \citep{Crammer:2001} by considering the exact labels, (b) Widrow-Hoff \citep{Widrow:1988} by posing it as a regression problem and (c) multi-class Perceptron \citep{Crammer2001}. For PRIL, at each time step, we compute the average of $L_{MAE}$ (defined in eq.(\ref{eq:MAE_Loss})). For PRank, Widrow-Hoff and MC-Perceptron, we find the average absolute error ($\frac{1}{T}\sum_{t=1}^T |\hat{y}_t-y_t|$). We repeat the process 100 times and average the instantaneous losses across the 100 runs. 

We train PRIL with 100\% partially labeled examples (Type 1 or Type 2). But when we compute the MAE, we use exact labels of examples. Which means, we used the model trained using PRIL with all partial labels and measure its performance with respect to the exact labels. This sets a harder evaluation criteria for PRIL. Moreover, in practice, we want a single predicted label and we want to see how similar it is as compared to the original label. Figure~\ref{fig:exp2} shows the comparison plot of PRIL with PRank, Widrow-Hoff (WH) and Multi-Class Perceptron (MCP). 

We see that PRIL does better as compared to Widrow-Hoff. This happens because Widrow-Hoff does not consider the categorical nature of the labels even though it respects the orderings of the labels. On the other hand, MCP is a complex model for solving a ranking problem. It also does not consider the ordering of the labels. We see that for Synthetic and Parkinsons datasets PRIL performs better that MCP. 

We observe that PRIL performance is comparable as compared to PRank \citep{Crammer:2001}. On Abalone and Parkinsons datasets, it performs similar to PRank. Which means that PRIL is able to recover the underlying classifier even if we have partial labels. This is a very interesting finding as we don't need to worry to provide exact labels. All we need is a range around the exact labels. Thus, PRIL appears to be a better way to deal with ordinal classification when we have uncertainties in the labels.

\section{Conclusions}
\label{sec:conclusions}
We proposed a new online algorithm called PRIL for learning ordinal classifiers when we only have partial labels for the examples. We show the correctness of the proposed algorithm. We show that PRIL converges after making finite number of mistakes whenever there exist an ideal partial labeling for every example. We also provide the mistake bound for general case. A regret bound is provided for PRIL. We also propose a multiplicative update algorithm for PRIL (M-PRIL). We show the correctness of M-PRIL and find its mistake bound. We experimentally show that PRIL is a very effective algorithm when we have partial labels.

\bibliography{Refereces_Ranking}

\begin{thebibliography}{12}
\providecommand{\natexlab}[1]{#1}
\providecommand{\url}[1]{\texttt{#1}}
\expandafter\ifx\csname urlstyle\endcsname\relax
  \providecommand{\doi}[1]{doi: #1}\else
  \providecommand{\doi}{doi: \begingroup \urlstyle{rm}\Url}\fi

\bibitem[Antoniuk et~al.(2015)Antoniuk, Franc, and Hlavac]{antoniuk14}
Antoniuk, Kostiantyn, Franc, Vojtech, and Hlavac, Vaclav.
\newblock Interval insensitive loss for ordinal classification.
\newblock In \emph{Proceedings of the Sixth Asian Conference on Machine
  Learning}, volume~39 of \emph{Proceedings of Machine Learning Research}, pp.\
   189--204, Nha Trang City, Vietnam, Nov 2015.

\bibitem[Antoniuk et~al.(2016)Antoniuk, Franc, and
  Hlav\'{a}\u{a}\'{z}]{Antoniuk:2016}
Antoniuk, Kostiantyn, Franc, Vojt\u{e}ch, and Hlav\'{a}\u{a}\'{z}, V\'{a}clav.
\newblock V-shaped interval insensitive loss for ordinal classification.
\newblock \emph{Machine Learning}, 103\penalty0 (2):\penalty0 261--283, May
  2016.

\bibitem[Chu \& Keerthi(2005)Chu and Keerthi]{Chu:2005}
Chu, Wei and Keerthi, S.~Sathiya.
\newblock New approaches to support vector ordinal regression.
\newblock In \emph{Proceedings of the 22Nd International Conference on Machine
  Learning}, ICML '05, pp.\  145--152, 2005.

\bibitem[Crammer \& Singer(2001{\natexlab{a}})Crammer and Singer]{Crammer2001}
Crammer, Koby and Singer, Yoram.
\newblock Ultraconservative online algorithms for multiclass problems.
\newblock In \emph{14th Annual Conference on Computational Learning Theory,
  COLT 2001}, pp.\  99--115, Amsterdam, The Netherlands, July
  2001{\natexlab{a}}.

\bibitem[Crammer \& Singer(2001{\natexlab{b}})Crammer and Singer]{Crammer:2001}
Crammer, Koby and Singer, Yoram.
\newblock Pranking with ranking.
\newblock In \emph{Proceedings of the 14th International Conference on Neural
  Information Processing Systems}, NIPS'01, pp.\  641--647, 2001{\natexlab{b}}.

\bibitem[Kivinen \& Warmuth(1997)Kivinen and Warmuth]{KIVINEN19971}
Kivinen, Jyrki and Warmuth, Manfred~K.
\newblock Exponentiated gradient versus gradient descent for linear predictors.
\newblock \emph{Information and Computation}, 132\penalty0 (1):\penalty0 1 --
  63, 1997.

\bibitem[Li \& Lin(2006)Li and Lin]{Li:2006}
Li, Ling and Lin, Hsuan-Tien.
\newblock Ordinal regression by extended binary classification.
\newblock In \emph{Proceedings of the 19th International Conference on Neural
  Information Processing Systems (NIPS)}, pp.\  865--872, 2006.

\bibitem[Lichman(2013)]{Lichman:2013}
Lichman, M.
\newblock {UCI} machine learning repository, 2013.
\newblock URL \url{http://archive.ics.uci.edu/ml}.

\bibitem[Rosenblatt(1958)]{rosenblatt}
Rosenblatt, F.
\newblock {The perceptron: A probabilistic model for information storage and
  organization in the brain.}
\newblock \emph{Psychological Review}, 65\penalty0 (6):\penalty0 386--408,
  1958.

\bibitem[Shashua \& Levin(2002)Shashua and Levin]{Shashua:2002}
Shashua, Amnon and Levin, Anat.
\newblock Ranking with large margin principle: Two approaches.
\newblock In \emph{Proceedings of the 15th International Conference on Neural
  Information Processing Systems}, NIPS'02, pp.\  961--968, Vancouver, British
  Columbia, Canada, 2002.

\bibitem[Widrow \& Hoff(1988)Widrow and Hoff]{Widrow:1988}
Widrow, Bernard and Hoff, Marcian~E.
\newblock Neurocomputing: Foundations of research.
\newblock chapter Adaptive Switching Circuits, pp.\  123--134. MIT Press, 1988.

\bibitem[Zinkevich(2003)]{Zinkevich:2003}
Zinkevich, Martin.
\newblock Online convex programming and generalized infinitesimal gradient
  ascent.
\newblock In \emph{Proceedings of the 12th International Conference on
  International Conference on Machine Learning}, ICML'03, pp.\  928--935, 2003.

\end{thebibliography}
\bibliographystyle{icml2016}

\end{document}